\documentclass[letterpaper, 10 pt, conference]{ieeeconf}  

\IEEEoverridecommandlockouts                              

\usepackage{graphics} 
\usepackage{epsfig} 
\usepackage{mathptmx} 
\usepackage{times} 
\usepackage{amsmath} 
\usepackage{amssymb}  
\usepackage{leftidx}
\usepackage{mathtools}
\usepackage{tensor}
\usepackage{theorem}
\usepackage[english, ruled, vlined]{algorithm2e}
\usepackage{mathtools}
\usepackage{thmtools}
\usepackage{thm-restate}
\usepackage{cleveref}
\usepackage{subfigure}
\declaretheorem[name=Theorem]{theorem}
\declaretheorem[name=Lemma]{lemma}
\declaretheorem[name=Definition]{definition}

\usepackage{authblk}
\usepackage{cite}

\newcommand\varpm{\mathbin{\vcenter{\hbox{%
  \oalign{\hfil$\scriptstyle+$\hfil\cr
          \noalign{\kern-.3ex}
          $\scriptscriptstyle({-})$\cr}%
}}}}

\title{\LARGE\bf Asymptotic Optimality of a Time Optimal\\ Path Parametrization Algorithm}
\author{Igor Spasojevic \and Varun Murali \and Sertac Karaman\thanks{Laboratory of Information and Decision Systems, Massachusetts Institute of Technology, 32 Vassar Street. This work was partly supported by the Office of Naval Research (ONR) and the Army Research Lab DCIST project. {\tt \{igorspas, mvarun, sertac\}@mit.edu}}}

\begin{document}

\maketitle
\thispagestyle{empty}
\pagestyle{empty}

\begin{abstract}

Time Optimal Path Parametrization is the problem of minimizing the time interval during which an actuation constrained agent can traverse a given path. Recently, an efficient linear-time algorithm for solving this problem was proposed \cite{TOPPRA}. However, its optimality was proved for only a strict subclass of problems solved optimally by more computationally intensive approaches based on convex programming. In this paper, we prove that the same linear-time algorithm is asymptotically optimal for all problems solved optimally by convex optimization approaches. We also characterize the optimum of the Time Optimal Path Parametrization Problem, which may be of independent interest. 

\end{abstract}

\section{Introduction}

Time optimal path parametrization is the problem of finding the shortest time required by a differentially constrained agent to execute a specified geometric path. For example, consider an autonomous race car that has to complete a given race track in minimum time. The path of the race car on the plane is fixed. Its speed along this path needs to be decided, given actuation constraints that, e.g., limit its acceleration. 


Seminal work on time optimal path parametrization dealt with planning trajectories for robotic manipulators \cite{BobrowDubowsky}. The first algorithm \cite{BobrowDubowsky}, enhanced thereafter in numerous works \cite{NIExtensions}, was theoretically grounded on Pontryagin's Maximum Principle \cite{OptimalControlBook}. This class of algorithms determined the optimal speed profile, a function mapping the position of the agent along the path to its speed, by stitching together integral curves arising from a bang-bang control policy. Although theoretically sound and computationally efficient, these algorithms were beset by issues of numerical instability \cite{TOPPRA}. 

More recent line of work \cite{LippBoyd}  \cite{VerscheureDemeulenaere} built on the insight that a whole spectrum of time optimal path parametrization problems could be cast as a convex optimization problem after a suitable change of variables. These algorithms solve for the optimal \textit{squared} speed profile. Under such reparametrization, typical constraints on velocity and acceleration of the agent, as well as the path traversal time, become convex functions of the decision variables.  
Specifically, convex optimization approaches first partition the path by a sequence of discretization points. They then jointly recover approximations of the optimal squared speed profile at every point \cite{LippBoyd} \cite{VerscheureDemeulenaere}. Although these methods are both numerically stable and converge to optimal solutions, their time complexity is considered to be high for many practical real-time robotics applications \cite{TOPPRA}. 

Similar to the convex-optimization-based approaches, algorithms developed in \cite{TOPPRA, ConsoliniDiscreteTime, ConvexWaiter, HungarianManipulators} approximate the optimal squared speed at a set of discretization points. In addition to being numerically stable, they are computationally efficient \cite{TOPPRA}, which they achieve by exploiting additional structure possessed by time optimal path parametrization problems. However, their optimality has been established for only a \textit{strict} subset of problems solved optimally by convex optimization approaches \cite{TOPPRA}.   


This paper shows the algorithm outlined in \cite{TOPPRA} is in fact \textit{optimal} in the limit as the distance between consecutive discretization points tends to zero. The main contribution of this paper is twofold. First, we develop a characterization of the optimal solution using results from non-smooth analysis \cite{kannan2012advanced} and non-linear control \cite{KhalilBook} that have not been previously used in this context to the best of our knowledge (Theorem \ref{ContinuousSupremum}). Second, we uncover a natural relationship between continuous solutions and those defined on a set of discretization points as output by all numerical algorithms (Theorem \ref{ConsistencyProof}).  

This paper is organized as follows. In section \ref{NonSmoothAnalysis}, we present the necessary background on non-smooth analysis used for the problem definition presented in Section \ref{ProbDef}. In Section \ref{OptimumChar}, we present our first main result that characterizes the optimal solution. We recall the algorithm given in \cite{TOPPRA} in Section \ref{NumAg}, and we present our second main result that proves the asymptotic optimality of this algorithm in Section \ref{AsympOpt}. 
\section{Preliminaries on Non-Smooth Analysis}
\label{NonSmoothAnalysis}
\begin{definition}\cite{KhalilBook,kannan2012advanced}
For a continuous function $h : [a, b] \rightarrow \mathbb{R}$, we define functions $D^{+}h, \ D^{-}h  :  [a, b) \rightarrow \mathbb{R} \cup \{\pm \infty\}$ given by
\[
D^{+}h (s) = \limsup_{s' \downarrow s} \frac{h(s') - h(s)}{s' - s}, \ D^{-}h (s) = \liminf_{s' \downarrow s} \frac{h(s') - h(s)}{s' - s},
\]
for all $s \in [a,b)$. Additionally, $h$ is called Dini differentiable if both $D^{+}h$ and $D^{-}h$ take on values strictly in $\mathbb{R}$. 
\end{definition}

By definition, $D^{+}h(s) \geq D^{-}h(s)$ for all $s \in [a,b)$. Furthermore, $h$ is right differentiable at $s$ if and only if $D^{+}h(s) = D^{-}h(s) \in \mathbb{R}$, in which case its right derivative equals $D^{+}h(s)$. For every pair of Dini differentiable functions $h_1$ and $h_2$, non-negative $\theta \in \mathbb{R}$, and right differentiable function $f$:

\begin{enumerate}
\item $D^{+}(h_1 + h_2) \leq D^{+}h_1 + D^{+}h_2$
\item $D^{-}(h_1 + h_2) \geq D^{-}h_1 + D^{-}h_2$
\item $D^{\pm}(\theta h_1) = \theta \ D^{\pm} h_1$
\item \label{DiniSignFlip} $D^{+}(-h_1) = - D^{-}h_1$
\item \label{DiniDifferentiableCombination} $D^{\pm}(h + f) = D^{\pm}h + f'$.
\end{enumerate}
\medskip
The following theorem is one of the fundamental results of non-smooth analysis. 

\begin{theorem} \cite{kannan2012advanced} \label{RoughMeanValueTheorem}
Let $h : [a, b] \rightarrow \mathbb{R}$ be a continuous function. The following are equivalent: 
\begin{enumerate}
\item $h$ is monotonically decreasing on $[a, b]$
\item $D^{-}h(s) \in [-\infty, 0]$ for all $s \in [a,b)$
\item $D^{+}h(s) \in [-\infty, 0]$ for all $s \in [a,b)$.
\end{enumerate}

\end{theorem}

Theorem \ref{RoughMeanValueTheorem} has two important corollaries. Firstly, Property (\ref{DiniSignFlip}) of Dini derivatives implies a continuous function $h$ is monotonically increasing if and only if $D^{+}h$ and $D^{-}h$ are non-negative functions. Second, if $D^{+}h$ is bounded above (below) by $\lambda \in \mathbb{R}$, property \ref{DiniDifferentiableCombination} implies $h(s') \leq (\geq) \ h(s) + \lambda (s' - s)$ for all $a \leq s \leq s' \leq b$.

\section{Problem Definition}\label{ProbDef}

A concrete example of time optimal path parametrization involves minimizing the time a car requires to traverse a specified smooth geometric path $\gamma : [0, S_{end}] \rightarrow \mathbb{R}^3$. For the sake of simplicity, we assume $\gamma$ is parametrized by arc length. Constraints consist of upper bounds on the magnitudes of velocity $v$ and acceleration $a$ of the car at every point along the path. Letting $h(s) := \left|\left|v(s)\right|\right|_2^2$ 
, we have \cite{LippBoyd, VerscheureDemeulenaere}: 
\begin{equation*}
\begin{aligned}
v(s) & = \frac{ds}{dt}\gamma'(s) = \sqrt{h(s)}\gamma'(s), \\
a(s) & = \frac{ds}{dt}\frac{d}{ds}\left(\sqrt{h(s)}\gamma'(s)\right) = \frac{1}{2}h'(s)\gamma'(s) + h(s)\gamma''(s).
\end{aligned}
\end{equation*}

The bound on velocity $\left|\left| v(s) \right|\right|_2 \leq v_{max}$ is equivalent to $h(s) \leq v_{max}^2$, while the bound on acceleration $|| a(s) ||_2 \leq F_{fr}$ translates into  

\begin{equation} \label{AccelerationMagnitudeBound}
|h'(s)| \leq 2\sqrt{F_{fr}^2 - ||\gamma''(s)||_2^2 h(s)^2}.
\end{equation}

%

For the very simple case of moving optimally along a straight line segment after starting from rest, the acceleration of the car switches from $+2 F_{fr}$ to $-2 F_{fr}$. At the switching point, the squared speed has discontinuous slope. To seamlessy deal with such behaviour, we drop the requirement that $h$ be a differentiable function and substitute Condition (\ref{AccelerationMagnitudeBound}) by $D^{+}h(s) \leq f^{+}(s, h(s))$ and $D^{-}h(s) \geq f^{-}(s, h(s))$ for suitably chosen functions $f^{+}$ and $f^{-}$.
In the most general form, we solve problem $P(B_u, B_l, f^{+}, f^{-})$:
\begin{equation} \label{ContinuousParameterProblem}
\begin{aligned}
& \underset{h : [a,b] \rightarrow [0,\infty)}{\text{minimize}}
& & \int_{a}^{b} \frac{ds}{\sqrt{h(s)}} \\
& \text{subject to}
& & D^{+}h(s) \leq f^{+}(s, h(s)), \  s \in [a, b), \\
&&& D^{-}h(s) \geq f^{-}(s, h(s)), \ s \in [a, b), \\
&&& B_l(s) \leq h(s) \leq B_u(s), \ s \in [a, b].
\end{aligned}
\end{equation}
The former example is clearly a special case of the latter problem, as can be seen by setting $a = 0$, $b = S_{end}$, $B_l \equiv 0$, $B_u = \min(v_{max}^2, F_{fr}^2 / ||\gamma''(s)||_2^2)$, and $f^{\pm}(s, h) = \pm 2 \sqrt{F_{fr}^2 - ||\gamma''(s)||_2^2 h^2}$. 

We note that if a pair of feasible solutions $h_1$ and $h_2$ satisfies $h_1(s) \leq h_2(s)$ for all $s \in [a, b]$, $h_2$ has a lower cost than $h_1$. 

\section{Characterization of Optimum}
\label{OptimumChar}

The main result of this section is presented in Theorem \ref{ContinuousSupremum}. We prove that the function defined as the pointwise supremum of all functions that are feasible for problem $P$ is also feasible and therefore optimal (Theorem~\ref{ContinuousSupremum}(a)). We use this characterization to show continuity of the optimum with respect to a natural parameter quantifying the degree of relaxation of constraints of $P$ (Theorem~\ref{ContinuousSupremum}(b)). Finally, we prove that the feasible set of $P$ is convex  (Theorem~\ref{ContinuousSupremum}(c)). To begin with, we note a useful result from Lipschitz analysis.    

\begin{theorem}\cite{LipschitzAnalysis}\label{LipschitzSuprema}
Let $\{ h_{\alpha} \}_{\alpha \in A}$ be an arbitrary non-empty family of uniformly bounded $\lambda$-Lipschitz functions defined on interval $[a,b]$. Functions $\overline{h}, \underline{h} : [a, b] \rightarrow \mathbb{R}$, defined by 
\[
\overline{h}(s) = \sup_{\alpha \in A} h_{\alpha}(s), \ \underline{h}(s) = \inf_{\alpha \in A} h_{\alpha}(s),
\]
for all $s \in [a,b]$, are well-defined and $\lambda$-Lipschitz.
\end{theorem}

\begin{theorem}\label{ContinuousSupremum}
Let $B_l, B_u : [a, b] \rightarrow \mathbb{R}$ be a pair of continuous functions with $B_u(s) \geq B_l(s)$ for all $s \in [a, b]$. Define region $F := \{ (s, h) \ \vert \ s \in [a, b], \ B_l(s) \leq h \leq B_u(s) \}$. Suppose $f^{+}, f^{-} : F  \rightarrow \mathbb{R}$ are a pair of continuous functions with $f^{+}(s, h) \geq f^{-}(s, h)$ for all $(s, h) \in F$. In particular, $|f^{\pm}| \leq B$ for some $B > 0$. For a real number $\xi \geq 0$, a function $h : [a, b] \rightarrow \mathbb{R}$ is called $\xi$-feasible if it satisfies the following conditions: 
\begin{enumerate}
\item \label{continuity_constraint} $h$ is continuous
\item \label{boundedness_constraint} $B_l(s) \leq h(s) \leq B_u(s)$ for all $s \in [a,b]$
\item \label{acceleration_lower_bound} $D^{-}h(s) \geq f^{-}(s, h(s)) - \xi$ for all $s \in [a, b]$
\item \label{acceleration_upper_bound} $D^{+}h(s) \leq f^{+}(s, h(s)) + \xi$ for all $s \in [a, b]$.
\end{enumerate}
Let $A_{\xi}$ denote the set of $\xi$-feasible functions. 
\begin{enumerate}

\item[a)] \label{closure_under_suprema} Assume $\emptyset \ne \{ h_{\alpha} \}_{\alpha \in C} \subseteq A_{\xi}$ for some (possibly uncountable) index set $C$. Then, $\overline{h}, \underline{h} : [a, b] \rightarrow \mathbb{R}$, defined by 
\[
\overline{h}(s) = \sup_{\alpha \in C} h_{\alpha}(s), \ \underline{h}(s) = \inf_{\alpha \in C} h_{\alpha}(s),
\]
for all $s \in [a, b]$, are $\xi$-feasible functions.
\item[b)] \label{continuity_wrt_perturbations} Assume $A_0 \ne \emptyset$. Define $\overline{h}_{\xi} = \sup_{h \in A_{\xi}} h$. Then,
\[
\left|\left| \overline{h}_{\xi} - \overline{h}_0  \right|\right|_{\infty} \rightarrow 0 \text{ as }  \xi \rightarrow 0.
\]

\item[c)] \label{feasible_convexity} Assume functions $f^{+}$ and $f^{-}$ are concave and convex in their second arguments respectively. For every $\xi \geq 0$, for every pair of $\xi$-feasible functions $h_1$ and $h_2$, and for every $\theta \in [0,1]$, the function $h_{\theta} = \theta h_1 + (1 - \theta) h_2$ is also $\xi$-feasible. 

\end{enumerate}

\end{theorem}

\begin{proof}
(a) We only give detailed proof of the claim for $\xi = 0$ and $\overline{h}$. The corresponding result for $\overline{h}$ when $\xi > 0$ can be recovered from the result for $\xi = 0$ by redefining $f^{\pm} \rightarrow f^{\pm} \pm \xi$. Similarly, the result for $\underline{h}$ can be recovered from the result for $\overline{h}$ by redefining $B_l \rightarrow -B_u$, $B_u \rightarrow -B_l$, $f^{\pm} \rightarrow -f^{\mp}$ and using Properties $(1)$-$(5)$ of Dini derivatives. 

Since functions $B_u$ and $B_l$ are continuous on $[a,b]$, they are bounded. As $C \neq \emptyset$, $\overline{h}$ is well defined. For arbitrary $s \in [a, b]$, taking the supremum over $\alpha \in C$ of the inequality $B_l(s) \leq h_{\alpha}(s) \leq B_u(s)$, we verify $\overline{h}$ satisfies Condition $(2)$. In particular, $f^{+}$ and $f^{-}$ are defined at all points $\left(s, \overline{h}(s) \right)$ for $s \in [a, b]$.

The fact that $|f^{\pm}| \leq B$ implies $h_{\alpha}$ is $B$-Lipschitz for all $\alpha \in C$. By Theorem \ref{LipschitzSuprema}, $\overline{h}$ is also $B$-Lipschitz; in particular $\overline{h}$ is continuous, thus verifying Condition $(1)$.
 

\begin{figure}[h!]
  \includegraphics[scale = 0.55]{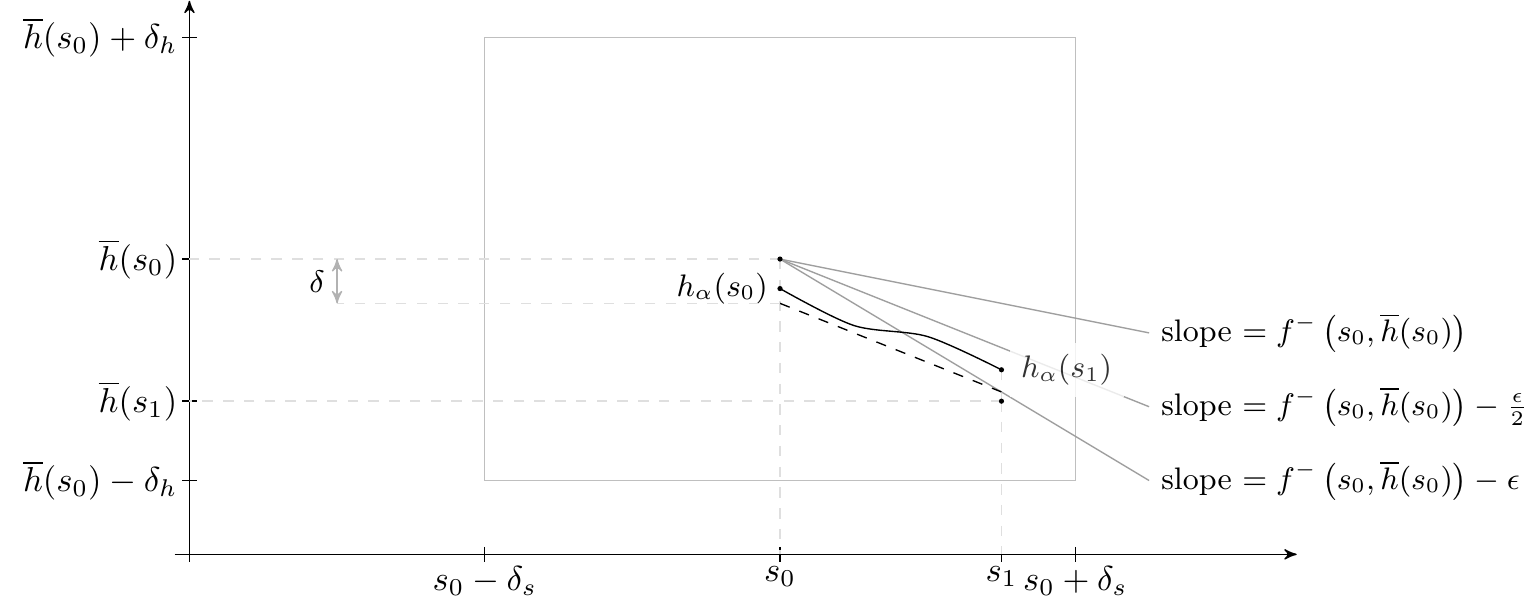}
  \caption{An illustration of the proof of Condition (3) of Theorem 3.}
  \label{fig:CharacterizationContradiction}
\end{figure}

Next, we show $\overline{h}$ satisfies Condition $(3)$. Assume, for a contradiction, $D^{-}\overline{h}(s_0) < f^{-}(s_0, \overline{h}(s_0))$ for some $s_0 \in [a, b)$ (see Figure  \ref{fig:CharacterizationContradiction}). Hence, there exists $\epsilon > 0$ such that $D^{-}\overline{h}(s_0) < f^{-} \left( s_0, \overline{h}(s_0) \right) - \epsilon$. 

Continuity of $f^{-}$ implies there exist $\delta_s, \delta_h > 0$ such that $f^{-}\left( s, h \right) \geq f^{-} \left( s_0, \overline{h}(s_0) \right) - \frac{\epsilon}{2}$ for all $(s, h) \in F \cap \left[ s_0 - \delta_s, s_0 + \delta_s \right] \times \left[ \overline{h}(s_0) - \delta_h, \overline{h}(s_0) + \delta_h \right]$. Keeping $\delta_h$ intact while shrinking $\delta_s$ if necessary, we may assume 
\begin{equation}
\delta_h > (B + \epsilon) \delta_s.
\end{equation}

The definition of $D^{-}$ implies there exists a decreasing sequence $(s_n)_{n \geq 1}$ tending to $s_0$ as $n \rightarrow \infty$ and $\frac{\overline{h}(s_n) - \overline{h}(s_0)}{s_n - s_0} \leq f^{-} \left( s_0, \overline{h}(s_0) \right) - \epsilon$ for all $n \geq 1$. In particular, we may assume $s_1 \in \left( s_0, s_0 + \delta_s \right)$ satisfies 
\begin{equation} \label{h_s_1_small}
\overline{h}(s_1) \leq \overline{h}(s_0) + \left( f^{-} \left( s_0, \overline{h}(s_0) \right) - \epsilon \right) \left( s_1 - s_0 \right).
\end{equation}
Define 
\begin{equation}
\delta = \min \left( \frac{\epsilon}{2}(s_1 - s_0), \delta_h - B \delta_s \right).
\end{equation}
The definition of $\overline{h}$ implies there exists $\alpha \in C$ such that 
\begin{equation}
\overline{h}(s_0) - \delta < h_{\alpha}(s_0) \leq \overline{h}(s_0). 
\end{equation}
Since $h_{\alpha}$ is $B$-Lipschitz, it follows that for all $s \in [s_0, s_1]$ we have 
\begin{equation}
\begin{aligned}
\left| h_{\alpha}(s) - \overline{h}(s_0) \right| & = \left| h_{\alpha}(s) - h_{\alpha}(s_0) + h_{\alpha}(s_0) - \overline{h}(s_0) \right| \\
&  \leq \left| h_{\alpha}(s) - h_{\alpha}(s_0) \right| + \left| h_{\alpha}(s_0) - \overline{h}(s_0) \right| \\
& \leq B(s - s_0) + \delta \\
& \leq B \delta_s + \delta_h - B \delta_s = \delta_h.
\end{aligned} 
\end{equation}
Hence, for all $s \in [s_0, s_1]$
\[
\left( s, h_{\alpha}(s) \right) \in F \cap \left[ s_0 - \delta_s, s_0 + \delta_s \right] \times \left[ \overline{h}(s_0) - \delta_h, \overline{h}(s_0) + \delta_h \right],
\]implying 
\begin{equation}
D^{-}h_{\alpha}(s) = f^{-} \left( s, h_{\alpha}(s) \right) \geq f^{-} \left( s_0, \overline{h}(s_0) \right) - \frac{\epsilon}{2}.
\end{equation} 
By the second corollary to Theorem \ref{RoughMeanValueTheorem}, 
\begin{equation} \label{condition_three_contradiction}
\begin{aligned}
h_{\alpha}(s_1) & \geq h_{\alpha}(s_0) + \left( f^{-} \left( s_0, \overline{h}(s_0) \right) - \frac{\epsilon}{2} \right) \left( s_1 - s_0 \right) \\
&  > \overline{h}(s_0) - \delta + \left( f^{-}(s_0, \overline{h}(s_0)) - \frac{\epsilon}{2} \right) \left( s_1 - s_0 \right) \\
&  \geq \overline{h}(s_0) - \left( s_1 - s_0 \right)\frac{\epsilon}{2} + \left( f^{-}(s_0, \overline{h}(s_0)) - \frac{\epsilon}{2} \right) \left( s_1 - s_0 \right) \\
& \geq \overline{h}(s_1),
\end{aligned}
\end{equation}
where the last inequality follows from Equation (\ref{h_s_1_small}). However, Equation (\ref{condition_three_contradiction}) violates the definition of $\overline{h}$ at $s_1$. This gives the desired contradiction, and shows $\overline{h}$ satisfies Condition (\ref{acceleration_lower_bound}). The proof that $\overline{h}$ satisfies Condition (\ref{acceleration_upper_bound}) is ommitted since it can be derived analogously.

\medskip

(b)
Consider arbitrary real numbers $0 \leq \xi_1 \leq \xi_2$. According to part (a) of Theorem \ref{ContinuousSupremum}, $\overline{h}_{\xi_1} \in A_{\xi_1} \subseteq A_{\xi_2}$. This implies $\overline{h}_{\xi_1} \leq \overline{h}_{\xi_2}$. Hence, for every $s \in [a, b]$, $\overline{h}_{\xi}(s)$ is monotonically increasing in $\xi \geq 0$ and bounded below by $\overline{h}_0(s)$. As a result, function $\tilde{h} : [a, b] \rightarrow \mathbb{R}$, given by $\tilde{h}(s) = \inf_{\xi  > 0} \overline{h}_{\xi}(s)$ for all $s \in [a, b]$, is well defined and satisfies $\tilde{h} \geq \overline{h}_0$.

On the other hand, monotonicity of $\overline{h}_{\xi}(s)$ implies $\tilde{h}(s) = \inf_{0 < \nu \leq \xi} \overline{h}_{\nu}(s)$ for every $\xi > 0$. Since $\overline{h}_{\nu} \in A_{\nu} \subseteq A_{\xi}$ for every $\nu \leq \xi$, another application of part (a) of Theorem \ref{ContinuousSupremum} to the non-empty set of functions $\left( \overline{h}_{\nu} \right)_{0 < \nu \leq \xi} \subseteq A_{\xi}$ yields $\tilde{h} \in A_{\xi}$. Since $\xi > 0$ was arbitrary, we have $\tilde{h} \in {\cap}_{\xi > 0} A_{\xi} = A_0$. By definition of $\overline{h}_0$, we thus have $\overline{h}_0 \geq \tilde{h}$. 

Combining previous observations, we get $\tilde{h} = \overline{h}_0$. Thus, $\overline{h}_{\xi}(s) \downarrow \overline{h}_0(s)$ as $\xi \downarrow 0$ for all $s \in [a, b]$. Since functions $(\overline{h}_{\xi})_{\xi \geq 0}$ are continuous on interval $[a, b]$, uniform convergence follows.

\medskip

(c)
Consider any $\xi \geq 0$ and $\theta \in [0, 1]$. Function $h_{\theta}$ clearly satisfies Conditions (\ref{continuity_constraint}) and (\ref{boundedness_constraint}) of Theorem \ref{ContinuousSupremum}, so we turn to deriving Condition (\ref{acceleration_lower_bound}). As in part (a), the proof of Condition (\ref{acceleration_upper_bound}) is ommitted as it can be derived analogously. We have: 
\begin{equation}
\begin{aligned}
D^{-} h_{\theta}(s) &= D^{-} (\theta h_1 + (1 - \theta) h_2)(s) \\
& \geq \theta D^{-}h_1(s) + (1 - \theta) D^{-}h_2(s) \\
& \geq \theta (f^{-}(s,h_1(s)) - \xi) + (1 - \theta) (f^{-}(s, h_2(s)) - \xi) \\
& \geq f^{-}(s, \theta h_1(s) + (1 - \theta) h_2(s)) - \xi \\
& = f^{-}(s, h_{\theta}(s)) - \xi.
\end{aligned}
\end{equation}

The first inequality above follows from Properties ($2$) and ($3$) of Dini derivatives, whereas the second inequality follows from $\xi$-feasibility of $h_1$ and $h_2$. Finally, the last inequality follows from convexity of $f^{-}$ in its second argument. 
\end{proof}

\section{Algorithm}
\label{NumAg}

In this section, we recall the algorithm presented in \cite{TOPPRA} for obtaining a numerical approximation to the optimal solution characterized in the previous section.    

We first recall standard concepts from numerical analysis, which will be used for describing and analyzing the algorithm. A \textit{discretization} $D = D([a, b], (s_i)_{i = 0}^n)$ of interval $[a, b]$ is an increasing sequence of points $(s_i)_{i = 0}^n$ satisfying $a = s_0 < ... < s_n = b$.  We denote its cardinality by $|D| = n+1$, and its resolution by $\Delta (D)  = \max_{1 \leq i \leq n} | s_i - s_{i-1} |$. 

\begin{algorithm}\label{BackwardForwardAlgorithm}
\SetAlgoLined
\KwData{$D = (s_i)_{i=0}^n$, $(B_l(s_i))_{i = 1}^n$, $(B_u(s_i))_{i = 1}^n$, $f^{+}$, $f^{-}$}
\KwResult{$(\hat{h}_i)_{i=0}^n$}

${h}_n^{(b)} = B_u(s_n)$\\
\For{$i = n-1$ to $0$}{
${h}_i^{(b)} \leftarrow \max \{h  \vert  h \leq B_u(s_i), h + f^{-}(s_i, h)(s_{i+1} - s_i) \leq {h}_{i+1}^{(b)}\}$\\
\If{${h}_i^{(b)} = -\infty$}{
	return null\\
}
}
$h_0^{(f)} = h_0^{(b)}$\\
\For{$i = 1$ to $n$}{
${h}_i^{(f)} \leftarrow \max \{ h  \vert  h \leq {h}_{i}^{(b)}, h \leq {h}_{i-1}^{(f)} + f^{+}(s_{i-1}, {h}_{i-1}^{(f)})(s_i - s_{i-1}) \}$ \\
\If{${h}_i^{(f)} = -\infty$}{
	return null\\
}
}
return $(\hat{h}_i)_{i = 0}^n = (h_i^{(f)})_{i = 0}^n$
\caption{Backward-Forward Algorithm}
\end{algorithm}

Given a discretization $D$ and problem $P(B_u, B_l, f^{+}, f^{-})$, a numerical procedure aims to find approximations $(\hat{h}_i)_{i = 0}^n$ to the optimal solution $\overline{h} = \overline{h}(P)$ at points $(s_i)_{i = 0}^n$. Its \textit{error} is defined as $\rho((\hat{h}_i)_{i = 0}^n, P, D) = \max_{0 \leq i \leq n}|\hat{h}_i - \overline{h}(s_i)|$, and it is said to be \textit{asymptotically optimal} if $\rho \rightarrow 0$ as $\Delta(D) \rightarrow 0$.  

Algorithm \ref{BackwardForwardAlgorithm} is a recently-proposed algorithm for solving problem $P$ numerically. It incrementally constructs an approximation to the optimum in a pair of sweeps through $(s_i)_{i = 0}^n$. As a result, it has linear \textit{time-complexity} in $|D|$. This makes it orders of magnitude faster than approaches employing general purpose convex optimization libraries, whose time complexity is super-linear in $|D|$ \cite{TOPPRA}. However, despite its computational efficiency, Algorithm \ref{BackwardForwardAlgorithm} had been proven to converge to optimal solutions for only a subclass of problems that can be optimally solved by convex programming in \cite{LippBoyd}.

\section{Asymptotic Optimality}
\label{AsympOpt}

The main result of this section is Theorem \ref{ConsistencyProof} which proves asymptotic optimality of Algorithm \ref{BackwardForwardAlgorithm} for all feasible problems $P$ amenable to convex optimization approaches. 

First, in Theorem \ref{ODETheorem} we recall an important result, which:
\begin{enumerate}
\item[a)] characterizes a lower bound on the length of the interval on which a solution to an ordinary differential equation is defined

\item[b)] proves that a continuous function can never exceed a differentiable function whose derivative upper bounds the former's Dini derivative.  
\end{enumerate}

\begin{theorem} \cite{KhalilBook} \label{ODETheorem}
In addition to the setup of Theorem \ref{ContinuousSupremum}, let:
\begin{enumerate}

\item $B_u$ and $B_l$ satisfy $B_u > B_l$

\item for every pair of continuous functions $U,L : [a, b] \rightarrow \mathbb{R}$ such that $B_l < L < U < B_u$, there exist $\lambda_s, \lambda_h > 0$ such that $f^{+}$ and $f^{-}$ are $\lambda_s$-Lipschitz and $\lambda_h$-Lipschitz on $\{(s, h) \vert s \in [a, b], L(s) \leq h \leq U(s)\}$ in their first and second arguments respectively.

\end{enumerate}

Consider arbitrary $g \in \{f^{+}, f^{-}\}$, $s_0 \in [a, b)$, and $h_0$ such that $L(s_0) < h_0 < U(s_0)$.

\begin{enumerate}

\item[a)] There exists $\delta > 0$ such that the initial value problem 
\[
h'(s) = g(s, h(s)) {\text{  subject to }} h(s_0) = h_0
\]
admits a unique solution on interval $[s_0, s_0 + \delta]$. Furthermore, we may choose 
\[
s_0 + \delta = \min \left( b, \inf\{s \geq s_0 \vert h(s) \notin (L(s), U(s))\} \right).
\]

\item[b)] Every continuous function $\tilde{h} : [s_0, s_0 + \tilde{\delta}] \rightarrow \mathbb{R}$, such that $L(s) < \tilde{h}(s) < U(s)$ and $D^{+}\tilde{h}(s) \leq g(s, \tilde{h}(s))$ for all $s \in [s_0, s_0 + \tilde{\delta})$, satisfies 
\[
\tilde{h}(s) \leq h(s)
\]
for all $s \in [s_0, s_0 + \min(\delta, \tilde{\delta})]$.

\end{enumerate}

\end{theorem} 

Before turning to the main result of the section, we give a definition. For a problem $P(B_u, B_l, f^{+}, f^{-})$ and discretization $D([a, b], (s_i)_{i = 0}^n)$, we call a sequence $(h_i)_{i = 0}^n$ \textit{admissible} if $B_l(s_i) \leq h_i \leq B_u(s_i)$ for all $0 \leq i \leq n$, and $f^{-}(s_i, h_i) \leq \frac{h_{i+1} - h_i}{s_{i+1} - s_i} \leq f^{+}(s_i, h_i)$ for all $0 \leq i \leq n-1$. Additionally, we will denote by $\overline{h}(P)$ ($\underline{h}(P)$) the pointwise supremum (infimum) of all feasible functions for $P$. 

\begin{theorem} \label{ConsistencyProof}
Assume in addition to the setup of Theorem \ref{ODETheorem}, problem $P(B_u, B_l, f^{+}, f^{-})$ is feasible and $\overline{h} := \overline{h}(P) > \underline{h}(P) =: \underline{h}$. For every $\epsilon > 0$, there exists an $\eta > 0$ such that for every discretization $D([a, b], (s_i)_{i = 0}^n)$ with resolution $\Delta(D) \leq \eta$, Algorithm \ref{BackwardForwardAlgorithm} returns an admissible sequence $(\hat{h}_i)_{i = 0}^n$ with $\rho((\hat{h}_i)_{i = 0}^n, P, D) < \epsilon$.
\end{theorem}

\begin{proof}
Fix an arbitrary $\epsilon > 0$. The proof will consist of two parts. We will show there exist $\eta_1, \eta_2 > 0$ such that for every discretization $D([a, b], (s_i)_{i = 0}^n)$ with resolution at most $\eta_1$ ($\eta_2$), Algorithm \ref{BackwardForwardAlgorithm} produces an admissible sequence $(\hat{h}_i)_{i = 0}^n$ satisfying $\hat{h}_i \geq \overline{h}(s_i) - \epsilon$ ($\hat{h}_i \leq \overline{h}(s_i) + \epsilon$) for all $0 \leq i \leq n$. Clearly setting $\eta = \min(\eta_1, \eta_2)$ yields proof of the theorem.  

To prove the first part, consider feasible functions $h_l$ and $h_u$ such that $\overline{h} - \epsilon < h_l < h_u < \overline{h}$. Such functions exist by assumption $\overline{h} > \underline{h}$ and part (c) of Theorem \ref{ContinuousSupremum}. Define $\delta_1 = \inf_{s \in [a, b]} (\overline{h} - h_u)$, $\delta_2 = \inf_{s \in [a, b]} (h_u - h_l)$, and $\delta_3 = \inf_{s \in [a, b]} (h_l - (\overline{h} - \epsilon))$. Clearly $\delta_1,\delta_2, \delta_3 > 0$.

Set $\delta = \min \left\{ \frac{\delta_1}{3}, \frac{\delta_2}{3}, \frac{\delta_3}{3} \right\}$. By assumption, there exist $\lambda_s, \ \lambda_h > 0$ such that for all 
$
(s_1, h_1), (s_2, h_2) \in G := \{(s,h) \vert s \in [a, b], h_l(s) - \delta \leq h \leq h_u(s) + \delta \},
$
we have 
\[
\left| f^{\pm}(s_1, h_1) - f^{\pm}(s_2, h_2) \right| \leq \lambda_s |s_2 - s_1| + \lambda_h |h_2 - h_1|. 
\]

Choose 
\begin{equation}\label{DefEta}
\eta_1 = \delta e^{-\lambda_h B (b-a)} \min \left\{ \frac{1}{2B}, \frac{B \lambda_h}{\lambda_s + B \lambda_h} \right\}.
\end{equation}

We claim that for any discretization $D([a, b], (s_i)_{i = 0}^n)$ with resolution $\Delta(D) \leq \eta_1$, Algorithm \ref{BackwardForwardAlgorithm} produces an admissible sequence $(\hat{h}_{i})_{i = 0}^n$ satisfying $\hat{h}_i \geq \overline{h}(s_i) - \epsilon$ for all $ 0 \leq i \leq n$.

The proof of the claim will proceed in two stages. The first will show the sequence $(h_i^{(b)})_{i = 0}^n$ generated by the backward pass satisfies $h_{i}^{(b)} \geq h_u(s_i) - \delta$ for all $0 \leq i \leq n$. The second will show the sequence $(h_i^{(f)})_{i = 0}^n$ generated by the forward pass satisfies $h_i^{(f)} \geq h_l(s_i) - \delta \geq \overline{h}(s_i) - \epsilon$ for all $0 \leq i \leq n$.  
 
\begin{figure}[h!]
  \includegraphics[scale=0.5]{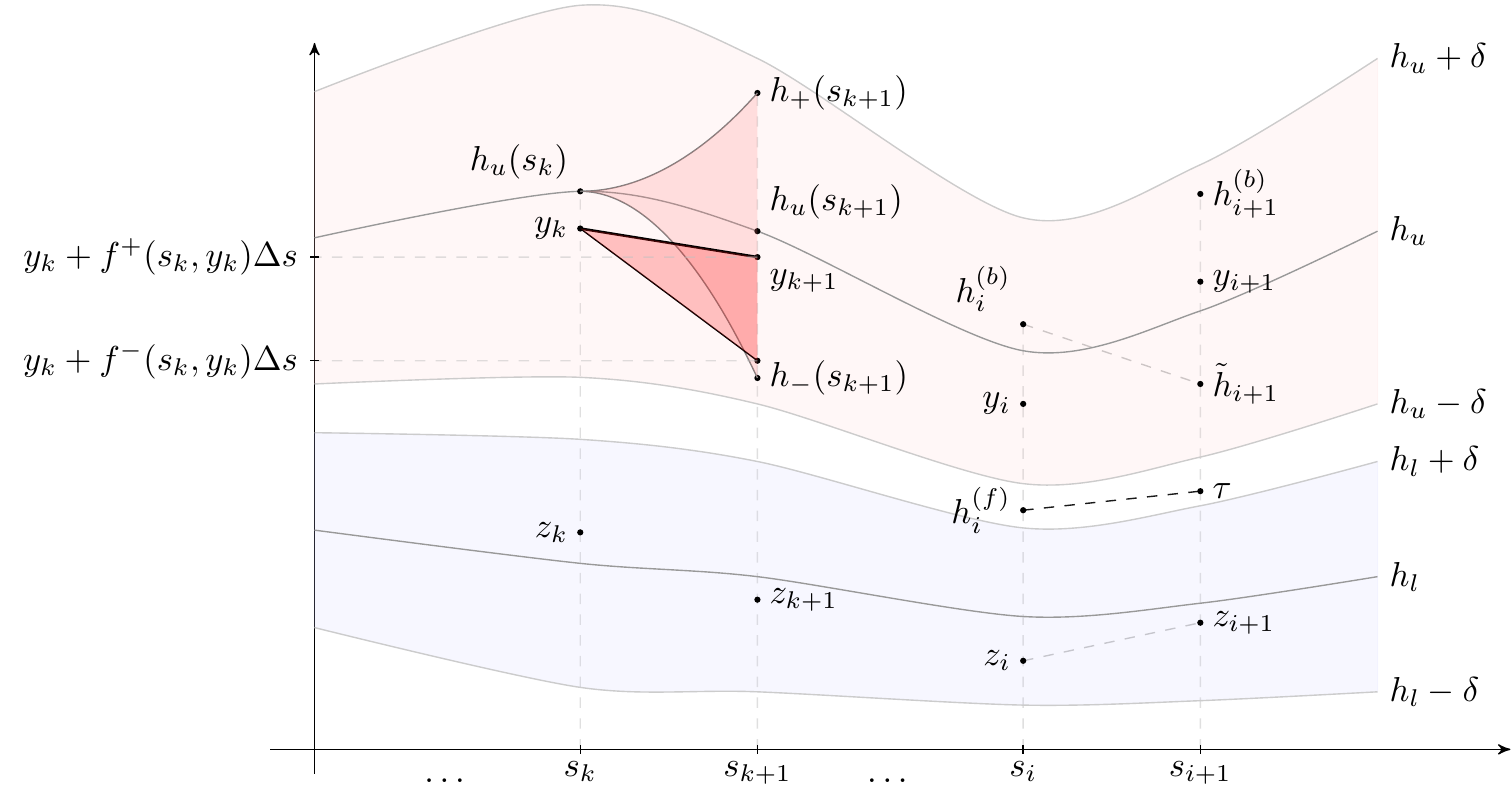}
  \caption{An illustration of the proof of Theorem \ref{ConsistencyProof} and Lemma \ref{BrokenLineLemma}.}
  \label{fig:LemmaProof}
\end{figure}
 
\begin{lemma} \label{BrokenLineLemma}
There exist admissible sequences $(y_k)_{k = 0}^n$ and $(z_k)_{k = 0}^n$ such that for every $0 \leq k \leq n$, we have $|y_k - h_u(s_k)| \leq \delta$ and $|z_k - h_l(s_k)| \leq \delta$ (see Figure \ref{fig:LemmaProof}). 
\end{lemma}

\begin{proof}
We only prove existence of $(y_k)_{k = 0}^n$ as that of $(z_k)_{k = 0}^n$ follows analogously. To this end, define $e : [a, b] \rightarrow \mathbb{R}$ given by 
\begin{equation}
e(s) = \delta e^{-\lambda_h B (b-a)} \left( e^{\lambda_h B (s-a)} - \frac{1}{2} \right)
\end{equation}
for all $s \in [a, b]$. Clearly $0 < e < \delta$. We will inductively construct an admissible sequence $(y_k)_{k = 0}^n$ which satisfies $|y_k - h_u(s_k)| \leq e(s_k)$, thus proving the lemma. 

Set $y_0 = h_u(s_0)$. Assume we have defined an admissible sequence $(y_j)_{j \leq k}$ satisfying $|y_j - h_u(s_j)| \leq e(s_j)$ for all $0 \leq j \leq k \leq n-1$. We now define $y_{k+1}$.   

By part (a) of Theorem \ref{ODETheorem}, the choice of $\eta_1$, along with boundedness and Lipschitz continuity of $f^{+}$ and $f^{-}$ on $G$, implies existence and uniqueness of solutions $h_{\pm}$ to initial value problems
\begin{equation} \label{LemmaIVP}
h_{\pm}'(s) = f^{\pm} (s, h_{\pm}(s)) \  \text{ subject to } \  h_{\pm}(s_k) = h_u(s_k)  
\end{equation}
on interval $[s_k, s_{k+1}]$. Furthermore, part (b) of Theorem \ref{ODETheorem} implies $h_{-}(s) \leq h_u(s) \leq h_{+}(s)$ for all $s \in [s_k, s_{k+1}]$. In particular, $h_{-}(s_{k+1}) \leq h_u(s_{k+1}) \leq h_{+}(s_{k+1})$.

Define $\Delta s = s_{k+1} - s_k$. Lipschitz continuity of $f^{\pm}$ and Equation (\ref{LemmaIVP}) imply (see Figure \ref{fig:LemmaProof}) 
\begin{equation} \label{exact_proximity}
\left| h_{\pm}(s_{k+1}) - ( h_{\pm}(s_k) + f^{\pm} (s_k, h_u(s_k)) \Delta s ) \right| \leq \frac{1}{2} \left(  \lambda_s + B \lambda_h \right) \Delta s^2.
\end{equation}
Similarly,
\begin{equation}
\begin{aligned} \label{approx_proximity}
 |  (y_k + & f^{\pm}(s_k, y_k) \Delta s)   - (h_{\pm}(s_k) + f^{\pm}(s_k, h_u(s_k))\Delta s) | \\ 
& \leq (1 + B \lambda_h \Delta s) \left|  y_k - h_u(s_k) \right|.
\end{aligned}
\end{equation}

Since an admissible value of $y_{k+1}$ can take on any value in the range $[y_k + f^{-}(s_k, y_k)\Delta s, y_k + f^{+}(s_k, y_k)\Delta s]$, Equations (\ref{exact_proximity}) and (\ref{approx_proximity}) imply the existence of admissible $y_{k+1}$ satisfying 
\begin{equation}
\begin{aligned}
|y_{k+1} - h_u(s_{k+1})| & \leq \frac{1}{2}(\lambda_s + B \lambda_h) \Delta s^2 + (1 + B \lambda_h \Delta s) |y_k - h_u(s_k)| \\
& \leq \frac{1}{2}(\lambda_s + B \lambda_h) \Delta s^2 + (1 + B \lambda_h \Delta s) e(s_k) \\
& \leq e(s_{k+1})
\end{aligned}
\end{equation}
where the second inequality follows from the inductive hypothesis, and the third from the definition of $e$ after a small amount of algebra. This completes proofs of the inductive step and the lemma.

\end{proof}

We now return to proofs of stages one and two. Assume sequences $\left( y_i \right)_{i = 0}^n$ and $\left( z_i \right)_{i = 0}^n$  have been constructed as in Lemma \ref{BrokenLineLemma}. Existence of $(y_i)_{i = 0}^n$ immediately implies the sequence $( h_i^{(b)})_{i = 0}^n$ is well-defined and satisfies $h_i^{(b)} \geq y_i \geq h_u(s_i) - \delta$ for all $0 \leq i \leq n$. This finishes the proof of stage one. 

For stage two, we prove by induction on $i$ that $h_i^{(f)}$ is well-defined and satisfies $z_i \leq h_i^{(f)} \leq h_i^{(b)}$ for all $0 \leq i \leq n$. The base case $i = 0$ follows from $h_0^{(f)} = h_0^{(b)} \geq y_0 > z_0$. For the inductive hypothesis, assume the statement holds for $i$. We now show it also holds for $i + 1$. The definition of the backward pass implies there exists $\tilde{h}_{i+1}$ such that 
\begin{equation} \label{DefBack}
h_{i+1}^{(b)} \geq \tilde{h}_{i+1} = h_i^{(b)} + f^{-}(s_i, h_i^{(b)})(s_{i+1} - s_i).
\end{equation}

We recall assumption $|f^{\pm}| \leq B$ along with feasibility of $h_u$ implies $h_u$ is $B$-Lipschitz. Thus,
\begin{equation}
\begin{aligned}
\tilde{h}_{i+1} & \geq h_i^{(b)} - B(s_{i+1} - s_i) \\
& \geq h_u(s_i) - \delta - B(s_{i+1} - s_i) \\
& \geq h_u(s_{i+1}) - \delta - 2 B (s_{i+1} - s_i).
\end{aligned}
\end{equation} 
Since $s_{i+1} - s_i \leq \frac{\delta}{2B}$, we have 
\begin{equation}\label{TildeBetween}
h_{i+1}^{(b)} \geq \tilde{h}_{i+1} \geq h_u(s_{i+1}) - 2 \delta \geq h_l(s_{i+1}) + \delta \geq z_{i+1}.
\end{equation}
Since $z_i \leq h_i^{(f)} \leq h_i^{(b)}$ (see Figure \ref{fig:LemmaProof}), there exists $\theta \in [0, 1]$ such that $h_i^{(f)} = \theta h_i^{(b)} + (1 - \theta) z_i$. Consider $\tau = \theta \tilde{h}_{i+1} + (1 - \theta) z_{i+1}$. Equation (\ref{TildeBetween}) implies
\begin{equation} \label{TauBetween} 
 z_{i+1} \leq \tau \leq h_{i+1}^{(b)}.
\end{equation}
Furthermore, 
\begin{equation}
\begin{aligned} \label{CandidateBigger}
\tau - h_i^{(f)} & = \theta (\tilde{h}_{i+1} - h_i^{(b)}) + (1 - \theta) (z_{i+1} - z_i) \\
& \geq \left( \theta f^{-}(s_i, h_i^{(b)}) + (1 - \theta) f^{-}(s_i, z_i) \right) (s_{i+1} - s_i) \\
& \geq f^{-}(s_i, \theta h_i^{(b)} + (1 - \theta) z_i) (s_{i+1} - s_i) \\
&  = f^{-}(s_i, h_i^{(f)}) (s_{i+1} - s_i),
\end{aligned}
\end{equation}
where the first inequality above follows from Equation (\ref{DefBack}) and admissibility of $(z_i)_{i = 0}^n$, and the second inequality from convexity of $f^{-}$ in its second argument. Similarly, we  obtain 
\begin{equation} \label{CandidateLess}
\tau - h_i^{(f)} \leq f^{+}(s_i, h_i^{(f)}) (s_{i+1} - s_i).
\end{equation} 
Equations (\ref{TauBetween}), (\ref{CandidateBigger}), and (\ref{CandidateLess}) imply $h_{i+1}^{(f)}$ is well-defined and satisfies $h_{i+1}^{(f)} \geq z_{i+1}$. This finishes the proof of the inductive step, the proof of stage two and of the first part of the theorem.

To prove the second part, consider $\xi > 0$ such that $\overline{h}_{\xi} \leq \overline{h} + \frac{\epsilon}{2}$. Such $\xi$ exists due to part (b) of Theorem \ref{ContinuousSupremum}. Uniform continuity of $f^{\pm}$ implies there exists 
\begin{equation} \label{Defdelta1}
\delta_1 \in (0, \epsilon / 2)
\end{equation}
such that for all $(s_1,h_1), (s_2, h_2) \in F$ we have 
\begin{equation} \label{fContinuity}
|s_1 - s_2| + |h_1 - h_2| \leq \delta_1 \Rightarrow | f^{\pm}(s_1,h_1) - f^{\pm}(s_2,h_2) | \leq \xi / 2.
\end{equation}
Uniform continuity of $B_u$ implies there exists 
\begin{equation} \label{BuContinuity}
\eta_2 < \frac{\delta_1}{1 + B}
\end{equation}
such that for all $s_1, s_2 \in [a, b]$ we have 
\begin{equation}\label{BuImplication}
|s_1 - s_2| \leq \eta_2 \Rightarrow |B_u(s_1) - B_u(s_2)| \leq \delta_1.
\end{equation}

Consider arbitrary discretization $D([a, b], (s_i)_{i = 0}^n)$ with $\Delta(D) \leq \eta_2$. Let $(\hat{h}_i)_{i = 0}^n$ be the sequence output by Algorithm \ref{BackwardForwardAlgorithm}. Define function $\tilde{h} : [a, b] \rightarrow \mathbb{R}$ via $\tilde{h}(s_i) = \hat{h}_i - \delta_1$ for all $0 \leq i \leq n$, and 
\begin{equation}
\tilde{h}(s) = \frac{s - s_i}{s_{i+1} - s_i}\tilde{h}(s_{i+1}) + \frac{s_{i+1} - s}{s_{i+1} - s_i}\tilde{h}(s_i)
\end{equation}
for all $s \in [s_i, s_{i+1}]$ and $0 \leq i \leq n-1$. By construction, $\tilde{h}$ is continuous. In fact, we show $\tilde{h} \in A_{\xi}$. 

First we will prove $\tilde{h} \leq B_u$. Consider any $0 \leq i \leq n-1$. Since $\tilde{h}$ is linear on $[s_i, s_{i+1}]$, $\tilde{h}(s) \leq \max (\tilde{h}(s_i), \tilde{h}(s_{i+1}))$ for all $s \in [s_i, s_{i+1}]$. Thus, it suffices to show $\tilde{h}(s_i), \tilde{h}(s_{i+1}) \leq \min_{s \in [s_i, s_{i+1}]} B_u(s)$. To this end, consider arbitrary $s \in [s_i, s_{i+1}]$. Since $|s - s_i| \leq \eta_2$, Equation (\ref{BuImplication}) implies $B_u(s) \geq B_u(s_i) - \delta_1$. Admissibility of $(\hat{h}_i)_{i = 0}^n$ implies $B_u(s_i) \geq \hat{h}_i$ and so
\begin{equation}
B_u(s) \geq B_u(s_i) - \delta_1 \geq \hat{h}_i - \delta_1 = \tilde{h}(s_i). 
\end{equation} 
Since $s$ was arbitrary, we obtain $\tilde{h}(s_i) \leq \min_{s \in [s_i, s_{i+1}]} B_u(s)$. The correspnding inequality for $\tilde{h}(s_{i+1})$ follows analogously, and so $\tilde{h} \leq B_u$ holds.

Next, we show $D^{+}\tilde{h}(s) \leq f^{+}(s, \tilde{h}(s))+ \xi$ for all $s \in [a, b)$. Again, consider arbitrary $0 \leq i \leq n-1$ and $s \in [s_i, s_{i+1})$. We have 
\begin{equation}
\begin{aligned}
D^{+}\tilde{h}(s) & = \frac{ \tilde{h}(s_{i+1}) -  \tilde{h}(s_i)}{s_{i+1} - s_i} 
= \frac{ \hat{h}_{i+1} -  \hat{h}_i}{s_{i+1} - s_i} \leq f^{+}(s_i, \hat{h}_i).
\end{aligned}
\end{equation} 
Also,
\begin{equation} \label{Proximity}
\begin{aligned}
|s - s_i| + |\tilde{h}(s) - \tilde{h}(s_i)|&  = 
|s - s_i| \left(1 + \left|  \frac{ h_{i+1} -  h_i}{s_{i+1} - s_i} \right| \right)  \\
& \leq |s - s_i| (1 + B) \leq \delta_1
\end{aligned}
\end{equation}
where the first equality follows from linearity of $\tilde{h}$ on $[s_i, s_{i+1}]$, and the second inequality from admissibility of $(\hat{h}_i)_{i = 0}^n$ and the fact $|f^{\pm}| \leq B$. Equations (\ref{fContinuity}) and (\ref{Proximity}) imply $f^{+}(s_i, \tilde{h}(s_i)) \leq f^{+}(s, \tilde{h}(s)) + \frac{\xi}{2}$. Similarly, $|\tilde{h}(s_i) - \hat{h}_i| \leq \delta_1$ implies $|f^{+}(s_i, \hat{h}_i) - f^{+}(s_i, \tilde{h}(s_i))| \leq \frac{\xi}{2}$. Combining the latter pair of inequalities, we derive $D^{+}\tilde{h}(s) \leq f^{+}(s, \tilde{h}(s))+ \xi$. Similarly, $D^{-}\tilde{h}(s) \geq f^{-}(s, \tilde{h}(s)) - \xi$, and so we obtain $\tilde{h} \in A_{\xi}$.

As a result, by definition of $\overline{h}_{\xi}$, we have $\tilde{h} \leq \overline{h}_{\xi}$. This implies for every $0 \leq i \leq n$
\begin{equation}
\begin{aligned}
h_i & = \tilde{h}(s_i) + \delta_1 
 \leq \overline{h}_{\xi}(s_i) + \delta_1 
 \leq  \overline{h}(s_i) + \frac{\epsilon}{2} + \delta_1 
 \leq \overline{h}(s_i) + \epsilon
\end{aligned}
\end{equation}
where the last inequality follows from Equation (\ref{Defdelta1}). This finishes the proof of the theorem.

\end{proof}

\section{Conclusion}

This paper presented two main results. First, it characterized the optimum of a large class of problems in time optimal path parametrization. Second, it proved the asymptotic optimality of a recently-proposed algorithm for solving this class of problems with linear (optimal) time complexity. This result extends its asymptotic optimality guarantee to all problems that are solved by relatively computationally more demanding convex-optimization-based methods. 
Let us note that, although we focused on the analysis of the algorithm presented in \cite{TOPPRA}, intermediate results in the proof of Theorem \ref{ConsistencyProof} could easily be combined to yield asymptotic optimality of the algorithm in \cite{ConvexWaiter}.

 

\bibliographystyle{IEEEtran}
\bibliography{references}

\begin{thebibliography}{10}
\providecommand{\url}[1]{#1}
\csname url@samestyle\endcsname
\providecommand{\newblock}{\relax}
\providecommand{\bibinfo}[2]{#2}
\providecommand{\BIBentrySTDinterwordspacing}{\spaceskip=0pt\relax}
\providecommand{\BIBentryALTinterwordstretchfactor}{4}
\providecommand{\BIBentryALTinterwordspacing}{\spaceskip=\fontdimen2\font plus
\BIBentryALTinterwordstretchfactor\fontdimen3\font minus
  \fontdimen4\font\relax}
\providecommand{\BIBforeignlanguage}[2]{{%
\expandafter\ifx\csname l@#1\endcsname\relax
\typeout{** WARNING: IEEEtran.bst: No hyphenation pattern has been}%
\typeout{** loaded for the language `#1'. Using the pattern for}%
\typeout{** the default language instead.}%
\else
\language=\csname l@#1\endcsname
\fi
#2}}
\providecommand{\BIBdecl}{\relax}
\BIBdecl

\bibitem{TOPPRA}
H.~Pham and Q.~C. Pham, ``A {N}ew {A}pproach to {T}ime-{O}ptimal {P}ath
  {P}arameterization {B}ased on {R}eachability {A}nalysis,'' \emph{IEEE
  Transactions on Robotics}, vol.~34, pp. 645 -- 659, 06 2018.

\bibitem{BobrowDubowsky}
J.~Bobrow, S.~Dubowsky, and J.~S.~Gibson, ``Time-{O}ptimal {C}ontrol of
  {R}obotic {M}anipulators {A}long {S}pecified {P}aths,'' \emph{International
  Journal of Robotic Research - IJRR}, vol.~4, pp. 3--17, 09 1985.

\bibitem{NIExtensions}
Q.~C. Pham, ``A {G}eneral, {F}ast, and {R}obust {I}mplementation of the
  {T}ime-{O}ptimal {P}ath {P}arameterization {A}lgorithm,'' \emph{IEEE
  Transactions on Robotics}, vol.~30, 12 2013.

\bibitem{OptimalControlBook}
L.~S.~Pontryagin, \emph{The {M}athematical {T}heory of {O}ptimal
  {P}rocesses}.\hskip 1em plus 0.5em minus 0.4em\relax CRC Press, 03 1987.

\bibitem{LippBoyd}
T.~Lipp and S.~Boyd, ``Minimum-time speed optimisation over a fixed path,''
  \emph{International Journal of Control}, vol.~87, 02 2014.

\bibitem{VerscheureDemeulenaere}
D.~Verscheure, B.~Demeulenaere, J.~Swevers, J.~De~Schutter, and M.~Diehl,
  ``Time-{O}ptimal {P}ath {T}racking for {R}obots: {A} {C}onvex {O}ptimization
  {A}pproach,'' \emph{Automatic Control, IEEE Transactions on}, vol.~54, pp.
  2318 -- 2327, 11 2009.

\bibitem{ConsoliniDiscreteTime}
L.~Consolini, M.~Locatelli, A.~Minari, and A.~Piazzi, ``An optimal complexity
  algorithm for minimum-time velocity planning,'' \emph{Systems \& Control
  Letters}, vol. 103, pp. 50--57, 05 2017.

\bibitem{ConvexWaiter}
G.~Csorv\'asi, A.~Nagy, and I.~Vajk, ``Near time-optimal path tracking method
  for waiter motion problem,'' vol.~50, 07 2017, pp. 4929--4934.

\bibitem{HungarianManipulators}
L.~Consolini, M.~Locatelli, A.~Minari, A.~Nagy, and I.~Vajk, ``Optimal
  time-complexity speed planning for robot manipulators,'' 02 2018.

\bibitem{kannan2012advanced}
R.~Kannan and C.~Krueger, \emph{Advanced Analysis: on the Real Line}, ser.
  Universitext.\hskip 1em plus 0.5em minus 0.4em\relax Springer New York, 2012.

\bibitem{KhalilBook}
H.~K. Khalil, \emph{{Nonlinear systems; 3rd ed.}}\hskip 1em plus 0.5em minus
  0.4em\relax Upper Saddle River, NJ: Prentice-Hall, 2002.

\bibitem{LipschitzAnalysis}
J.~Heinonen, ``Lectures on {L}ipschitz analysis,'' \emph{Rep. Dept. Math.
  Stat}, vol. 100, 01 2005.

\end{thebibliography}

\end{document}